\documentclass{article}

\usepackage[table]{xcolor}

\usepackage[accepted]{main}

\usepackage{amsmath}
\usepackage{amsfonts}
\usepackage{bm}
\usepackage{amssymb}
\usepackage{mathtools}
\usepackage{amsthm}
\usepackage{dsfont}
\usepackage{bbm} 
\usepackage{stmaryrd} 
\usepackage{enumitem}
\usepackage{todonotes}

\definecolor{rightblue}{RGB}{76,114,176} 
\definecolor{rightorange}{RGB}{221,132,82} 
\definecolor{aliceblue}{rgb}{0.94, 0.97, 1.0} 
\definecolor{darkcerulean}{rgb}{0.03, 0.27, 0.49} 
\definecolor{iris}{rgb}{0.35, 0.31, 0.81} 
\definecolor{carmine}{rgb}{0.59, 0.0, 0.09} 
\definecolor{green(munsell)}{rgb}{0.0, 0.66, 0.47} 
\definecolor{celadon}{rgb}{0.67, 0.88, 0.69} 
\definecolor{bluerow}{rgb}{0.0, 0.53, 0.74} 
\definecolor{lightorange}{RGB}{255, 219, 187} 
\definecolor{lavenderblue}{rgb}{0.8, 0.8, 1.0}
\definecolor{blue(pigment)}{rgb}{0.2, 0.2, 0.6}
\definecolor{blue-violet}{rgb}{0.54, 0.17, 0.89}
\usepackage{hyperref}
   \hypersetup{
     final=true,
     plainpages=false,
     pdfstartview=FitV,
     pdftoolbar=true,
     pdfmenubar=true,
     bookmarksopen=true,
     bookmarksnumbered=true,
     breaklinks=true,
     linktocpage=true,
     colorlinks=true,
     linkcolor=blue-violet, 
     urlcolor=iris, 
     citecolor=darkcerulean, 
     anchorcolor=black
}

\usepackage{url}            
\usepackage{nicefrac}       
\usepackage{microtype}      

\usepackage{caption} 
\def\thick{0.8} 
\usepackage{subcaption}

\usepackage{booktabs}
\usepackage{multirow}
\usepackage{graphicx}

\usepackage[capitalize]{cleveref}

\usepackage{tikz}
\usepackage{varwidth}

\usepackage{stmaryrd}

\usepackage{dirtytalk}
\MakeRobust{\say} 

\usepackage{makecell}

\usepackage{wrapfig}

\usepackage{enumitem}

\usepackage{nicematrix}

\usepackage{mathrsfs}

\def\thick{0.8}

\usepackage{overpic}

\usepackage{lipsum}

\usepackage{soul}

\definecolor{mygreen}{HTML}{5a9027}

\newcommand{\Loss}{\mathcal{L}}
\newcommand{\fm}{f_{\mathrm{FM}}}
\newcommand{\enc}{\mathrm{enc}}
\newcommand{\dec}{\mathrm{dec}}

\newcommand{\moment}{\texttt{Moment} }
\newcommand{\moirai}{\texttt{Moirai}\xspace}
\newcommand{\chronos}{\texttt{Chronos}\xspace}
\newcommand{\pca}{\texttt{PCA}\xspace}
\newcommand{\vae}{\texttt{VAE}\xspace}
\newcommand{\adapts}{\textbf{\texttt{AdaPTS}}\xspace}

\newcommand{\iid}{$\mathrm{iid}$\xspace}

\newcommand{\RR}{\mathbb{R}}

\newcommand{\mbf}[1]{\mathbf{#1}}

\newcommand{\bA}{\mathbf{A}}

\newcommand{\bX}{\mathbf{X}}
\newcommand{\bW}{\mathbf{W}}

\newcommand{\bY}{\mathbf{Y}}
\newcommand{\bZ}{\mathbf{Z}}

\newcommand{\vx}{{\bf x}}
\newcommand{\vy}{{\bf y}}

\newcommand{\vb}{{\bf b}}

\newcommand{\bB}{{\bf B}}

\newcommand{\RNum}[1]{\uppercase\expandafter{\romannumeral #1\relax}}

\usepackage{mleftright} 

\usepackage{xspace}

\newcommand{\norm}[1]{\left\lVert#1\right\rVert}

\DeclareMathOperator*{\argmin}{arg\,min}

\theoremstyle{plain}
\newtheorem{theorem}{Theorem}[section]
\newtheorem{proposition}[theorem]{Proposition}

\theoremstyle{definition}
\newtheorem{definition}[theorem]{Definition}
\newtheorem{assumption}[theorem]{Assumption}
\theoremstyle{remark}
\newtheorem{remark}[theorem]{Remark}

\icmltitlerunning{\adapts: Adapting Univariate Foundation Models to Probabilistic Multivariate Time Series Forecasting}

\allowdisplaybreaks

\newcommand\rebuttal[1]{#1}

\begin{document}

\addtocontents{toc}{\protect\setcounter{tocdepth}{0}}

\twocolumn[
\icmltitle{\adapts: Adapting Univariate Foundation Models to Probabilistic Multivariate Time Series Forecasting}

\begin{icmlauthorlist}
\icmlauthor{Abdelhakim Benechehab}{huawei,eurecom}
\icmlauthor{Vasilii Feofanov}{huawei}
\icmlauthor{Giuseppe Paolo}{huawei}
\icmlauthor{Albert Thomas}{huawei}
\icmlauthor{Maurizio Filippone}{kaust}
\icmlauthor{Bal\'{a}zs K\'{e}gl}{huawei}
\end{icmlauthorlist}

\icmlaffiliation{huawei}{Huawei Noah's Ark Lab, Paris, France}
\icmlaffiliation{eurecom}{Department of Data Science, EURECOM}
\icmlaffiliation{kaust}{Statistics Program, KAUST}

\icmlcorrespondingauthor{Abdelhakim Benechehab}{abdelhakim.benechehab@gmail.com}

\vskip 0.3in
]

\printAffiliationsAndNotice{} 

\begin{abstract}
Pre-trained foundation models (FMs) have shown exceptional performance in univariate time series forecasting tasks. However, several practical challenges persist, including managing intricate dependencies among features and quantifying uncertainty in predictions. This study aims to tackle these critical limitations by introducing \textbf{adapters}—feature-space transformations that facilitate the effective use of pre-trained univariate time series FMs for multivariate tasks. Adapters operate by projecting multivariate inputs into a suitable latent space and applying the FM independently to each dimension. Inspired by the literature on representation learning and partially stochastic Bayesian neural networks, we present a range of adapters and optimization/inference strategies. Experiments conducted on both synthetic and real-world datasets confirm the efficacy of adapters, demonstrating substantial enhancements in forecasting accuracy and uncertainty quantification compared to baseline methods. Our framework, \textbf{AdaPTS}, positions adapters as a modular, scalable, and effective solution for leveraging time series FMs in multivariate contexts, thereby promoting their wider adoption in real-world applications. We release the code at \href{https://github.com/abenechehab/AdaPTS}{https://github.com/abenechehab/AdaPTS}.
\end{abstract}

\section{Introduction}
\label{sec:intro}

\begin{figure}[t]
     \centering
     \begin{subfigure}[b]{\columnwidth}
         \centering
        \includegraphics[width=\textwidth]{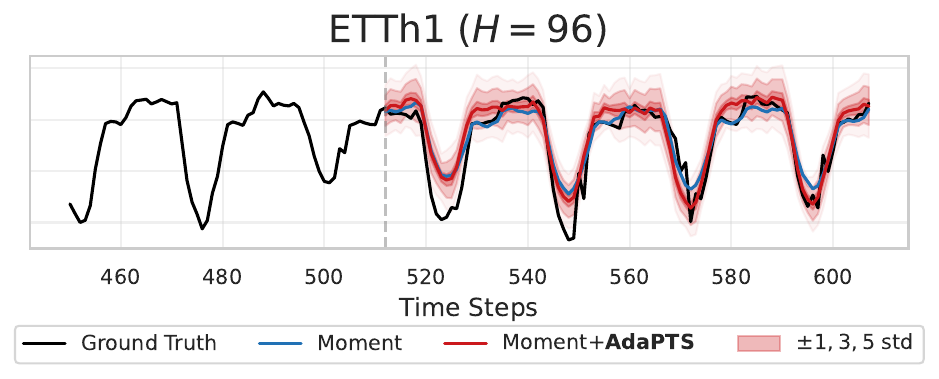}
         \caption{}
         \label{fig:traj}
     \end{subfigure}
     \vfill
     \begin{subfigure}[b]{\columnwidth}
         \centering
        \includegraphics[width=\textwidth]{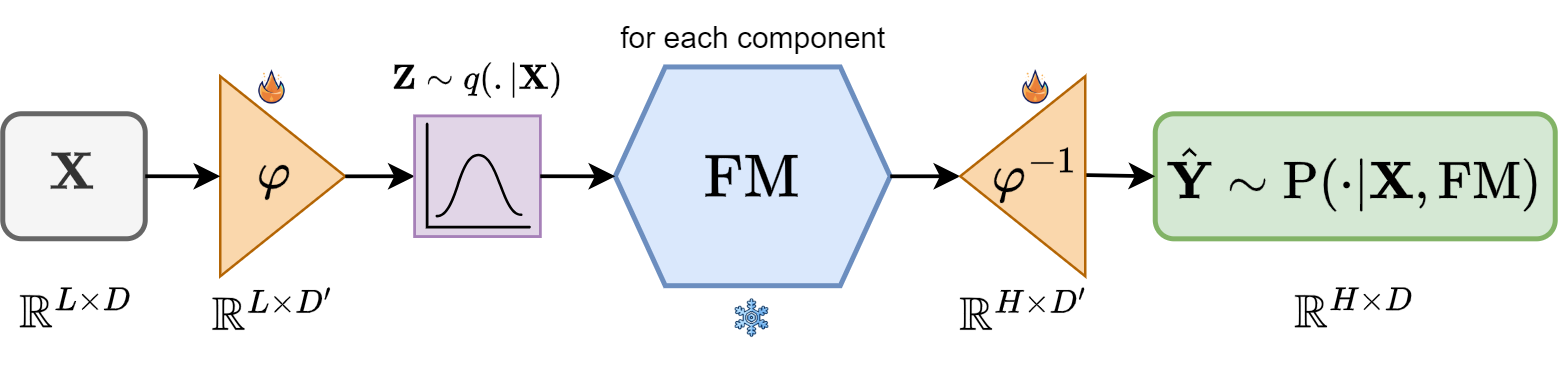}
         \caption{}
         \label{fig:pipeline}
     \end{subfigure}
\caption{(a) Augmenting $\moment$ time series foundation model with the \adapts framework provides \emph{probabilistic} and \emph{more accurate} predictions. (b) \textbf{The \adapts framework:} The input time series is transformed through a feature space transformation $\varphi$ that maps into a stochastic latent space. The prediction is then conducted using a pre-trained FM before transforming back the predicted, now distribution, to the original feature space. The fire symbol indicate trainable weights while the snowflake implicates that the parameters of the FM are kept frozen.
}
\label{fig:main}
\end{figure}

Time series forecasting is a well-established machine learning problem that involves analyzing sequential data to predict future trends based on historical patterns. Two key challenges frequently arise in this context: (a) time series are often multivariate, incorporating multiple descriptive features~\citep{wei2019multivariate}, and (b) estimating the uncertainty of a forecast is equally important, requiring probabilistic model outputs~\citep{gneiting2014probabilistic}. These challenges are particularly relevant in real-world applications where risk assessment depends on reliable forecasts, such as healthcare~\citep{jones2012improved}, finance~\citep{groen2013real}, energy management~\citep{zhang2014review,nowotarski2018recent}, and weather prediction~\citep{palmer2012towards,bi2023accurate}.  

Existing foundation models (FMs) for time series forecasting, such as \chronos~\citep{ansari2024chronos}, are typically trained for univariate forecasting tasks due to tractability constraints, as the wide range of real world time series problems typically have different numbers of features. Even without discretization, handling multivariate time series directly within these models (\moment~\citep{goswami2024moment}, \moirai~\citep{liu2024moirai}) remains computationally challenging due to the high-dimensional dependencies among features. This limitation raises a fundamental question: how can we leverage existing pre-trained univariate FMs to enable probabilistic forecasting for multivariate time series?  

To address this, we introduce \adapts, a novel framework designed to augment FMs with probabilistic adapters. As illustrated in Figure~\ref{fig:main}, \adapts applies a stochastic feature transformation that maps the input time series into a latent space, where predictions are made using the frozen FM. Our framework sets itself apart from existing literature by enforcing an invertibility constraint on the adapter, allowing predictions to be transformed back into the original feature space. Beyond enhancing forecasting accuracy, the integration of stochasticity into the adapter's latent representation ensures that the model captures uncertainty, thereby improving both calibration and robustness.

Our approach leads to several novel insights and contributions, which we summarize as follows:  

\begin{enumerate}
    \item \textbf{Multivariate FM adaptation.} We introduce a principled methodology for adapting existing pre-trained univariate FMs to multivariate probabilistic forecasting, resulting in the \adapts framework. 
    \item \textbf{Theoretical foundations of adapters.} We provide a theoretical analysis to support the necessity of adapters, starting with the analytically tractable case of linear adapters and linear FMs. We then build on the literature on partially stochastic Bayesian neural networks to introduce probabilistic adapters.  
    \item \textbf{Empirical validation.} We conduct extensive experiments on multivariate time series forecasting benchmarks, demonstrating that our approach improves forecasting accuracy baseline methods. We also analyze the interpretability of the learned latent representation and show that adapters enable cost-effective adaptation by reducing the dimensionality of the feature space. 
\end{enumerate}

The rest of this paper is organized as follows: ~\cref{sec:related} discusses related work, ~\cref{sec:adapters} details the problem setup and the theoretical analysis on linear adapters. ~\cref{sec:prob_adapters} extends our framework to probabilistic adapters, and ~\cref{sec:exps} showcases experimental results. Finally, we conclude with limitations and future directions in ~\cref{sec:discussion}.

\section{Related work}
\label{sec:related}

\noindent\textbf{Time Series Foundational Models.} Over the past two years, a plethora of foundation models have been proposed with a particular focus on time series forecasting. Some of these models like $\texttt{GPT4TS}$~\citep{zhou2023onefitsall} and $\texttt{Time-LLM}$~\citep{jin2023time} \rebuttal{``reprogram''} a Large Language Model to the forecasting setting by freezing most of its layers and fine-tuning additional time series-specific modules to a new downstream task. The majority of these time series FMs including \texttt{Lag-Llama}~\citep{rasul2024lagllama}, $\chronos$~\citep{ansari2024chronos}, $\moirai$~\citep{liu2024moirai}, \texttt{TimesFM} \citep{das2024a} and \moment~\citep{goswami2024moment} are trained from scratch on a large volume of time series data.

\noindent\textbf{Adapters.} The multivariate setting presents a significant challenge for time series FMs, as different tasks involve varying numbers of channels\footnote{Throughout the paper, the words \emph{features}, \emph{channels}, and \emph{components} are used interchangeably to refer to the number of variates in a multivariate time series, represented as $D$ in \cref{sec:adapters}.}. To the best of our knowledge, the only model that naturally accommodates any number of channels is $\moirai$~\citep{liu2024moirai}, which, however, suffers from high computational demand due to processing all channels flattened in the transformer simultaneously, leading to a quadratic memory complexity w.r.t. to the number of channels. Most foundation models, instead, treat each one of these independently, which, as noted by~\citet{feofanov2024adapters}, remains computationally expensive when full fine-tuning is required. For classification tasks with hundreds or thousands of features, they demonstrated that simple adapters like the rotation matrix obtained through Principal Components Analysis ($\pca$) mitigate this issue. At the same time,~\citet{benechehab2025zeroshot} showed that $\pca$ preserves channel interactions by learning new disentangled components. However, in both cases, $\pca$ provided little improvement over independent processing, leaving room for further enhancements. In the context of tabular regression, foundation models such as \citep[$\texttt{TabDPT}$]{ma2024tabdptscalingtabularfoundation} also use \pca to adapt to a variable number of features. 

Less related to our work, \citet{cheng2016_gpadapter} use a Gaussian process adapter in the context of irregular time series classification. In other domains, adapters have been used for multimodal (text-time series) representation learning~\citep{zhang2024dualtimedualadaptermultimodallanguage} and computer vision~\citep{LI2025, yin2023adapterneedtuningvisual, pan2022}.





    
    

\section{Adapters}
\label{sec:adapters}

\subsection{Problem setup}

Consider a multivariate long-term time series forecasting task, represented by: a data matrix $\bX \in \RR^{L \times D}$ where $L$ is the context window size and $D$ is the multivariate time series dimensionality, and a target matrix $\bY \in \RR^{H \times D}$, where $H$ is the forecasting horizon. 
We denote by $\vx_d \in \RR^{L \times 1}$ (respectively $\vy_d \in \RR^{H \times 1}$) the $d$-th component of the input (respectively target) multivariate time series.

Our goal is to use a frozen pre-trained univariate time series foundation model denoted as $\fm: \RR^{L \times 1} \to \RR^{H \times 1}$ (\cref{fig:pipeline}) and exploit the information stored in its weights to achieve the best forecasting performance, measured by the mean squared error (MSE) loss:

\begin{equation}
\label{eq:metric}
\Loss = \| \bY - \fm(\bX)\|_{\mathrm{F}}^2
\end{equation}

On multivariate time series, for simplicity, we denote by $\fm(\bX)$ the application of $\fm$ to each channel independently, in which case the loss can be written as: $\frac{1}{D} \sum_{d=1}^D \|\vy_d - \fm(\vx_d)\|_2^2$.

We now formally define an adapter, a tool by means of which we aim to best use the foundation model $\fm$ for multivariate forecasting:

\begin{definition}[adapter]
\label{def:adapter}

An adapter is a feature-space transformation $\varphi: \RR^{D} \to \RR^{D'}$ that is applied to the data prior to the foundation model\footnote{In practice, $\varphi$ is applied on matrices $\bX$ in $\RR^{L \times D}$. This denotes the application of $\varphi$ on each row of $\bX$.}. The forecast is then obtained by transforming the predictions back to the original feature space:
\[\hat{\bY}(\bX; \varphi) = \varphi^{-1} \big( \fm(\varphi(\bX)) \big)\]

\end{definition}

According to this definition, an adapter is valid only if the inverse transformation $\varphi^{-1}: \RR^{D'} \to \RR^D$, such that $\forall \vx \in \RR^D$, $\varphi^{-1} \circ \varphi (\vx) = \vx$, is well-defined on $\RR^{D'}$. In the rest of the paper, we relax this condition by naming the direct transformation as \emph{encoder} ($\varphi \triangleq \enc$), and respectively, the inverse transformation as \emph{decoder} ($\varphi^{-1} \triangleq \dec$). In this case, the predictions obtained after the application of the adapter become: $\hat{\bY}(\bX; \enc, \dec) = \dec \big( \fm(\enc(\bX)) \big)$.

We note that in the literature, there exist alternatives to adapt to the multivariate setting \citep{zhang2023crossformer,zhou2023onefitsall}, but we have chosen this family of adapters due to their high flexibility as: (a) any foundation model can be plugged-in, (b) no requirement of fine-tuning due to feature-level transformations~\citep{feofanov2024adapters}, (c) adaptation to the computation budget by defining the number of encoded channels.

\noindent\textbf{Optimality of an adapter.} In order for an adapter to be useful, it has to achieve a lower forecasting error than the identity baseline. In fact, the loss defined in \cref{eq:metric} corresponds to the forecasting loss obtained by using an adapter implementing the identity matrix $\mathbf{I}$. Therefore, we define the optimality of the adapter based on improving the forecasting error of the identity baseline:

\[ \Loss \geq \Loss(\varphi) = \|\bY - \varphi^{-1} \big( \fm(\varphi(\bX)) \big)\|_{\mathrm{F}}^2 \]

\subsection{Theoretical analysis}
\label{subsec:theory}

The purpose of this section is to study the optimization problem that the adapter $\varphi$ is aiming to solve:

\begin{equation}
\label{eq:prob1}
    \varphi^* = \argmin_\varphi \|\bY - \varphi^{-1} \big( \fm(\varphi(\bX)) \big)\|_{\mathrm{F}}^2
\end{equation}

Under mild assumptions on the adapter function class and the backbone foundation model $\fm$, we aim at characterizing the optimal solution $\varphi^*$ and prove that it realizes the optimality condition: $\Loss(\varphi^*) \leq \Loss$.

We first consider the linear case where we constrain the adapter $\varphi$ to the class of linear transformations, parametrized by a matrix $\bW_{\varphi} \in \RR^{D \times D}$: $\varphi(\bX) = \bX \bW_{\varphi}$.

\begin{assumption}
\label{ass:invertible}
$\bW_{\varphi}$ has full rank: $\text{rank}(\bW_{\varphi}) = D$, insuring its invertibility.
\end{assumption}

\begin{assumption}
\label{ass:fm}
For ease of derivation, we consider a similar linear parametrization for the foundation model: $\fm(\bX) =  \bW_{\text{FM}}^\top \bX + \vb_{\text{FM}} \mbf{1}^\top$ where $\bW_{\text{FM}} \in \RR^{L \times H}$, $\vb_{\text{FM}} \in \RR^H$, and $\mbf{1}$ a vector of ones of dimension $D$.
\end{assumption}

\begin{proposition}[Optimal linear adapter]
\label{prop:solution}
Under \cref{ass:invertible} and \cref{ass:fm}, the \emph{closed-form} solution of the problem: 

\begin{equation}
\label{eq:prob2}
    \Loss (\bW_{\varphi}) = \|\bY - \big( \bW_{\text{FM}}^\top \bX \bW_{\varphi} + \vb_{\text{FM}} \mbf{1}^\top \big) \bW_{\varphi}^{-1}\|_F^2
\end{equation}

writes as:

\begin{equation}
\label{eq:solution}
    \bW_{\varphi}^* = (\bB^\top \bA)^{+} \bB^\top \bB
\end{equation}

where $\bW_{\varphi}^* = \argmin_{\bW_{\varphi} \in \mathcal{GL}_D(\RR)} \Loss (\bW_{\varphi})$, $\bA = \bY -  \bW_{FM}^\top \bX$, $\bB = \vb_{FM} \mbf{1}^\top$, and $(\bB^{\top} \bA)^{+}$ denoting the \emph{pseudo-inverse} operator.

\end{proposition}

\begin{proof}
    The result follows by differentiating $\Loss (\bW_{\varphi})$ with respect to $\bW_{\varphi}$, and solving the Euler equation: $\nabla_{\bW_{\varphi}} \Loss = 0$. The detailed proof is deferred to \cref{appendix:theory}.
\end{proof}

\begin{remark}
In this case, the fact that the matrix $\bB = \vb_{FM} \mbf{1}^\top$ have identical columns renders the matrix $\bB^\top \bA$ degenerate (with $rank(\bB^\top \bA) = 1$). In practice, we add a positive constant to the diagonal in order to numerically stabilize the matrix inversion: $\bW_{\varphi}^* = (\bB^\top \bA + \lambda \mbf{I})^{-1} \bB^\top \bB$, with $\lambda > 0$. In \cref{subsec:motiv} we show that we are able to reach an optimal solution regardless of this added regularization.
\end{remark}

\subsection{Working example}
\label{subsec:motiv}

\noindent\textbf{Synthetic data.} \rebuttal{Our synthetic dataset comprises two multivariate time series, one with several independent components and the other with linearly correlated ones (\cref{fig:linear_synthetic})}, designed to evaluate a linear feature-space transformation. The data generation process creates five (\emph{uncorrelated}) base signals—sinusoids with distinct frequencies, amplitudes, and \rebuttal{\iid} noise— and derives eight additional channels through linear combinations of these bases with additive Gaussian noise of different magnitude ($\sigma \in (0.1, 0.2, 0.5)$). This construction provides a controlled environment where the ground truth relationship between channels is known: the underlying data manifold is effectively five-dimensional, \rebuttal{but in the correlated case} the observed eight-dimensional multivariate time series includes varying levels of noise and linear mixing. 

\noindent\textbf{Randomly generated linear FMs.} The experimental setup in \cref{fig:linear_synthetic} consists in randomly sampling the linear parameters of a toy foundation model: $\bW_{\text{FM}}$ and $b_{\text{FM}}$. To simulate a realistic scenario, we use Glorot-uniform initialization distribution as it would be the case in neural network-based architectures. We then compute the closed-form solution $\bW_{\varphi}^*$ \rebuttal{(\cref{eq:solution})} on raw data $\bX$, and compare the resulting loss value with the baseline (using the identity matrix $\mathbf{I}$ as adapter) and the $\pca$-only adapter. 

\begin{figure}[ht]
\centering
\includegraphics[width=\columnwidth]{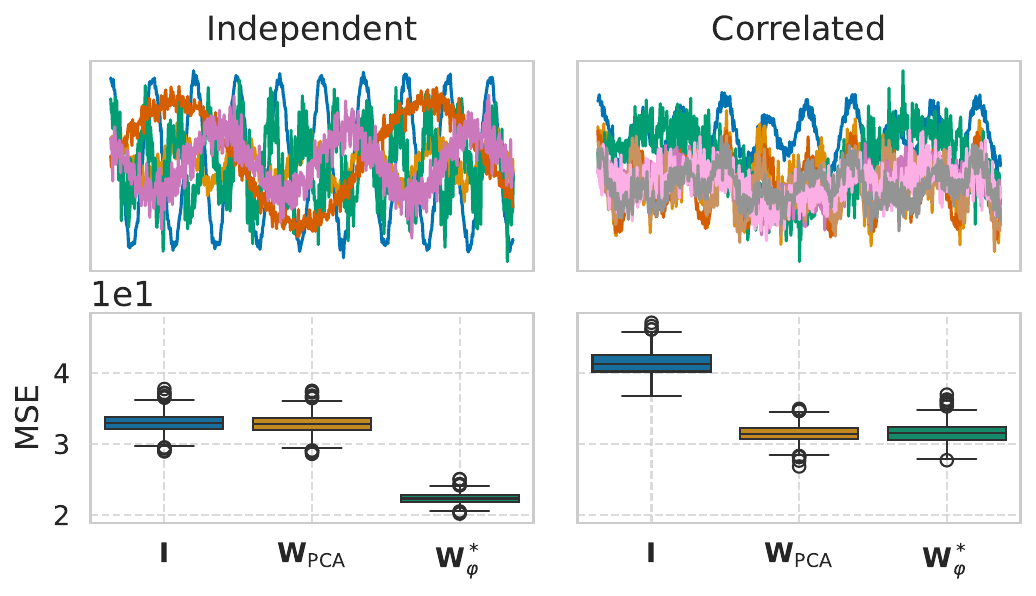}
\caption{\textbf{Optimality of $\bW_{\varphi}^*$}. Comparing the MSE obtained with $\bW_{\varphi}^*$ against the baseline, for 1000 randomly generated linear FM.}
\label{fig:linear_synthetic}
\end{figure}

\cref{fig:linear_synthetic} shows that in the case of uncorrelated data (\emph{left} column) $\pca$ is equivalent to the identity matrix, while the solution $\bW_{\varphi}^*$ to the problem $\Loss (\bW_{\varphi})$ reaches an order of magnitude better forecasting loss. In the correlated case, we observe that $\pca$ has a similar performance to the optimal solution. This example motivates the adapter idea through the existence of better linear transformations than the identity matrix in the case of linear foundation models.

\section{\adapts: Adapters for Probabilistic Multivariate Time Series Forecasting}
\label{sec:prob_adapters}

In this section, we introduce families of adapters that verify the conditions stated in \cref{def:adapter}. Furthermore, we extend this definition to include probabilistic variants of adapters, useful for uncertainty quantification on the FM predictions.

\subsection{Families of adapters}
\label{subsec:families}

\rebuttal{Our framework accommodates various families of transformations that can serve as adapters. Initially, we define linear AutoEncoders and subsequently extend them to their deep non-linear counterparts. Ultimately, we introduce the probabilistic adapter framework, encompassing Variational AutoEncoders (\vae) and $\texttt{Dropout}$-based families of adapters.}


\noindent\textbf{Linear AutoEncoders.} \rebuttal{In addition to the linear setup introduced in \cref{subsec:theory},} we extend Linear AutoEncoders to provide a simple yet effective method for dimensionality reduction while preserving the temporal relationships within time series data. In this more general case, the encoder compresses the multivariate time series $ \bX $ into a potentially lower-dimensional representation $ \bZ = \bX \bW_{\theta_{\enc}} $, where $ W \in \mathbb{R}^{D \times D'} $ is the linear transformation matrix, and $D' \leq D$. The decoder reconstructs the forecast to the original feature space after prediction as $ \hat{\bY} = \fm(\bZ) \bW_{\theta_{\dec}} $. Finally, the parameters of the encoder $\theta_{\enc}$ and the decoder $\theta_{\dec}$ are jointly optimized to minimize the objective in \cref{eq:prob1}.

\noindent\textbf{Deep non-linear AutoEncoders.} Deep non-linear AutoEncoders extend their linear counterparts by employing multiple layers of non-linear transformations. The encoder maps the input $ \bX $ to a latent space $ \bZ = \enc(\bX; \theta_{\enc}) $, where $ \enc $ is parameterized by a deep neural network. Similarly, the decoder reconstructs the predictions of the foundation model in the latent space:  $ \hat{\bY} = \dec(\fm(\bZ); \theta_{\dec})$.

Besides AutoEncoders, Normalizing Flows~\citep{Kobyzev_2021} such as \texttt{RealNVP}~\citep{dinh2017densityestimationusingreal} are a valid choice in the context of adapters, thanks to their inherently invertible nature. However, their training may be challenging due to various optimization related concerns. We defer a discussion on Normalizing Flows as adapters to \cref{appendix:flows}.

\subsection{Probabilistic Adapters}

We now discuss an alternative to the optimization of adapters, which is based on a Bayesian treatment of their parameters. 
There are many options on how to carry out inference over these parameters, and we can draw from the literature on Bayesian inference for neural networks \citep{PapamarkouICML24}. 

Considering a FM which yields point predictions, the appeal of Bayesian adapters is that they enable probabilistic predictions, which can be used for uncertainty quantification.
Note that this is the case for models such as \chronos and \moirai, which output a distribution over the time series continuous values\footnote{In the case of \chronos, this distribution is obtained through a categorical distribution (with \emph{softmax} probabilities) over a tokenized space of the time series values.}.
For deterministic FMs such as $\moment$ \citep{goswami2024moment}, a Bayesian treatment of adapters yields an ensemble of such predictions, which is key for accounting for the predictive uncertainty.

\noindent\textbf{Inference.} Recalling that $\theta$ represents the set of parameters of encoder ($\enc_\theta$) and decoder ($\dec_\theta$), we can attempt to obtain the posterior distribution over these parameters through Bayes theorem \cite{Gelman13}:
\begin{equation*}
p(\theta | \bY, \bX) \propto p(\bY | \bX, \theta) p(\theta)
\text{,}
\end{equation*}
where $p(\theta)$ is the prior distribution over the parameters and $p(\bY | \bX, \theta)$ the likelihood, with $\bY, \bX$ representing a training dataset in this context.
Alternatively, we can rather treat the latent representation $\bZ$ as stochastic, where the interest is now to characterize the following posterior: 
\begin{equation*}
p(\bZ | \bY, \bX) \propto p(\bY | \bX, \bZ) p(\bZ) \text{.}
\end{equation*}

In these two formulations, the posterior distribution over the parameters, is instrumental in obtaining predictive distributions useful for uncertainty quantification.
For instance, in the case of inference over $\theta$, 
for new test data $\bY^{*}, \bX^{*}$ we obtain:

\begin{equation*}
p(\bY^{*} | \bX^{*}, \bY, \bX) = \int p(\bY^{*} | \bX^{*}, \theta) p(\theta | \bY, \bX) d\theta \text{.}
\end{equation*}

Characterizing the posterior analytically, however, is intractable and we need to resort to approximations. 
The literature on Bayesian inference offers various strategies, which can be adapted to neural networks \citep{PapamarkouICML24}, including variational inference \citep{Graves11}, Laplace approximations \citep{Yang24}, and Markov chain Monte Carlo (MCMC) \citep{Chen14,Tran22}.

Within the \adapts framework, we focus in particular on variational inference (VI) for \vae adapters and on Monte Carlo dropout \citep{Gal16} as an approximate form of VI for carrying out inference over $\theta$. 

\noindent\textbf{Variational AutoEncoders.} Following the Bayesian perspective on adapters, $\vae$ assume a prior distribution over the latent representation $ \bZ $, typically $ \mathcal{N}(0, \mathbf{I}) $. The encoder then outputs parameters of the posterior distribution $ q_\phi(\bZ|\bX) $, and in our context, the decoder generates reconstructions of predictions $\hat{\bY} \sim p_\theta(\bY|\bX,\fm(\bZ)) $ where $\theta$ parametrize a likelihood model $p$. We then define the training objective of the $\vae$, which brings together the forecasting loss and a regularization term, in a similar way to the \emph{evidence lower bound} (ELBO)~\citep{Kingma2013AutoEncodingVB} objective:

\begin{proposition}[$\vae$ adapter training objective]
\label{prop:vae}
The training objective for the $\vae$ adapter is the maximization of an \emph{ELBO}-like lower bound on the marginal likelihood of the target $\bY$:
\begin{equation*}
\begin{split}
\log p_\theta(\bY|\bX, \fm) \geq & \mathbb{E}_{q_\phi(\bZ|\bX)} \left[ \log p_\theta(\bY|\bX, \fm(\bZ)) \right] \\
& - \mathrm{KL}\left(q_\phi(\bZ|\bX) \,\|\, p(\bZ)\right),
\end{split}
\end{equation*}
where $\mathrm{KL}$ denotes the Kullback-Leibler divergence. 
\end{proposition}

The derivation of this lower bound and a discussion on the implications of each term of the loss are deferred to~\cref{appendix:vae_proof}.

\begin{remark}
In practice, we use the Gaussian likelihood as our likelihood model: $p_\theta(\bY|\bX, \fm(\bZ)) = \mathcal{N}(\bY; \hat{\bY}, \sigma^2 \mathbf{I})$, with $\hat{\bY} = \dec_\theta(\fm(\bZ))$. In this case the forecasting loss term boils down to the MSE objective in \cref{eq:prob1} up to a multiplicative and additive noise-related constants: $\log \mathcal{N}(\bY; \hat{\bY}, \sigma^2 \mathbf{I}) = -\frac{1}{2\sigma^2} \| \bY - \hat{\bY} \|^2_{\mathrm{F}} - \frac{H D}{2} \log(2\pi\sigma^2)$. Notice that one can also learn a model of the noise where $ \dec_\theta(\fm(\bZ)) = [\mu_\theta(\bY|\bX, \fm(\bZ)), \sigma_\theta(\bY|\bX, \fm(\bZ))]$.
\end{remark}

\begin{remark}
The $\mathrm{KL}$ divergence regularization term can be multiplied by a scaling factor $\beta$ to control the disentanglement—independence of the latent representation components. This results in $\beta$-\vae~\citep{higgins2017betavae}, which is what we use in practice while referring to it as the \vae adapter throughout the paper. 
\end{remark}

\noindent\textbf{Dropout as approximate VI.} Dropout~\citep{srivastava14a} can be interpreted as a form of variational inference, where a variational distribution is imposed over the weights of a neural network \citep{Gal16}. Specifically, applying dropout during training corresponds to approximating a posterior over the weights using a Bernoulli distribution. This perspective allows the deterministic models introduced in \cref{subsec:families}, such as Linear AutoEncoders, to be transformed into probabilistic models by introducing stochasticity through dropout. 


Treating adapters in a Bayesian manner while keeping the FM fixed aligns with the concept of partially stochastic Bayesian neural networks, which provides theoretical guarantees on universal conditional density estimation \citep{Sharma23}. This framework ensures that the model can approximate any conditional density, provided that stochasticity is introduced early enough in the architecture and that the number of stochastic units matches or exceeds the output dimension. Using probabilistic adapters, We comply with these conditions by making the encoder stochastic, allowing the learned latent space to capture uncertainty while leveraging the FM’s fixed parameters.


\begin{table*}[t]
    \centering
    \begin{tabular}[t]{ccccccccc}
        \toprule[\thick pt]
        \multirow{2}{*}{Dataset} & \multirow{2}{*}{H} & \multicolumn{1}{c}{No adapter} & \multicolumn{5}{c}{with adapter} \\
        \cmidrule(lr){3-3} 
        \cmidrule(lr){4-8}
         & & \moment & \pca & $\texttt{LinearAE}$ & $\texttt{dropoutLAE}$ & $\texttt{LinearVAE}$ & \vae \\
        \midrule[\thick pt]
        \multirow{2}{*}{ETTh1} & 96 & $0.411_{\pm 0.012}$ & $0.433_{\pm 0.001}$ & $0.402_{\pm 0.002}$ & $\mathbf{0.395_{\pm 0.003}}$ & $0.400_{\pm 0.001}$ & $0.404_{\pm 0.001}$ \\
         & 192 & $\mathbf{0.431_{\pm 0.001}}$ & $0.440_{\pm 0.000}$  & $0.452_{\pm 0.002}$ & $0.446_{\pm 0.001}$ & $0.448_{\pm 0.002}$ & $\mathbf{0.431_{\pm 0.001}}$  \\
        \cmidrule(lr){1-8}
        \multirow{2}{*}{Illness} & 24 & $2.902_{\pm 0.023}$ & $2.98_{\pm 0.001}$ & $2.624_{\pm 0.035}$ & $2.76_{\pm 0.061}$ & $2.542_{\pm 0.036}$ & $\mathbf{2.461_{\pm 0.008}}$ \\
         & 60  & $3.000_{\pm 0.004}$ & $3.079_{\pm 0.000}$ & $3.110_{\pm 0.127}$ & $2.794_{\pm 0.015}$ & $\mathbf{2.752_{\pm 0.040}}$ & $2.960_{\pm 0.092}$\\
         \cmidrule(lr){1-8}
        \multirow{2}{*}{Weather} & 96 & $0.177_{\pm 0.010}$ & $0.176_{\pm 0.000}$ & $0.169_{\pm 0.000}$ & $\mathbf{0.156_{\pm 0.001}}$ & $0.161_{\pm 0.001}$ & $0.187_{\pm 0.001}$ \\
         & 192 & $0.202_{\pm 0.000}$ & $0.208_{\pm 0.001}$ & $\mathbf{0.198_{\pm 0.001}}$ & $0.200_{\pm 0.001}$ & $0.204_{\pm 0.000}$ & $0.226_{\pm 0.000}$  \\
        \cmidrule(lr){1-8}
        \multirow{2}{*}{ExchangeRate} & 96 & $\mathbf{0.130_{\pm 0.011}}$ & $0.147_{\pm 0.000}$ & $0.167_{\pm 0.013}$ & $\mathbf{0.130_{\pm 0.011}}$ & $0.243_{\pm 0.039}$ & $0.455_{\pm 0.010}$ \\
         & 192 & $\mathbf{0.210_{\pm 0.002}}$ & $0.222_{\pm 0.000}$ & $0.304_{\pm 0.005}$ & $0.305_{\pm 0.013}$ & $0.457_{\pm 0.020}$ & $0.607_{\pm 0.021}$ \\
        \bottomrule[\thick pt]
    \end{tabular}
    \caption{Performance comparison between the baseline $\moment$ model without adapters against different adapter architectures (\pca, $\texttt{LinearAE}$, $\texttt{dropoutLinearAE}$, $\texttt{LinearVAE}$, and \vae), for multivariate long-term forecasting with different horizons $H$. We display the average test MSE $\pm$ standard error obtained on $3$ runs with different seeds. \textbf{Best} results are in bold, with lower values indicating better performance.}
    \label{table:results}
\end{table*}

\section{Experiments \& Results}
\label{sec:exps}

In this section, we empirically demonstrate the quantitative and qualitative superiority of \adapts in multivariate long-term time series forecasting on common benchmarks. We show that in most of the considered tasks (datasets and forecasting horizons,) our framework improves the performance of \moment, a commonly used time series forecasting foundation model. The implementation details are provided in~\cref{appendix:implem}.

\subsection{Time series forecasting}

\noindent\textbf{Datasets.} Our experiments are conducted on four publicly available real-world multivariate time series datasets, commonly used for long-term forecasting \citep{ilbert2024samformer, wu2021autoformer, chen2023tsmixer, nie2023a, Zeng2022AreTE}. These datasets include the Electricity Transformer Temperature dataset (ETTh1) \citep{zhou2021informer}, ExchangeRate \citep{lai2018modelinglongshorttermtemporal}, Weather \citep{weather}, and Influenza-like Illness \citep{illness}. All time series are segmented with an input length of $L = 512$, prediction horizons $H \in [96, 192]$ and $H \in [24, 60]$ for the Illness dataset, and a stride of 1, meaning each subsequent window is shifted by one step. These datasets (detailed in \cref{appendix:datasets}) originate from various application domains, enabling a comprehensive evaluation of our framework across diverse real-world scenarios.

\noindent\textbf{Baseline.} We compare our method against the vanilla application of the foundation model \moment$_{\text{small}}$ from the $\moment$ family of models \citep{goswami2024moment}. This means that for each dataset, we apply \moment$_{\text{small}}$ independently to each feature. Additionally, we compare our learning-based adapters against \pca, an adapter that has been used in the literature for model-based reinforcement learning \citep{benechehab2025zeroshot} and time series classification \citep{feofanov2024adapters}.


\begin{figure}[ht]
\centering
\includegraphics[width=\columnwidth]{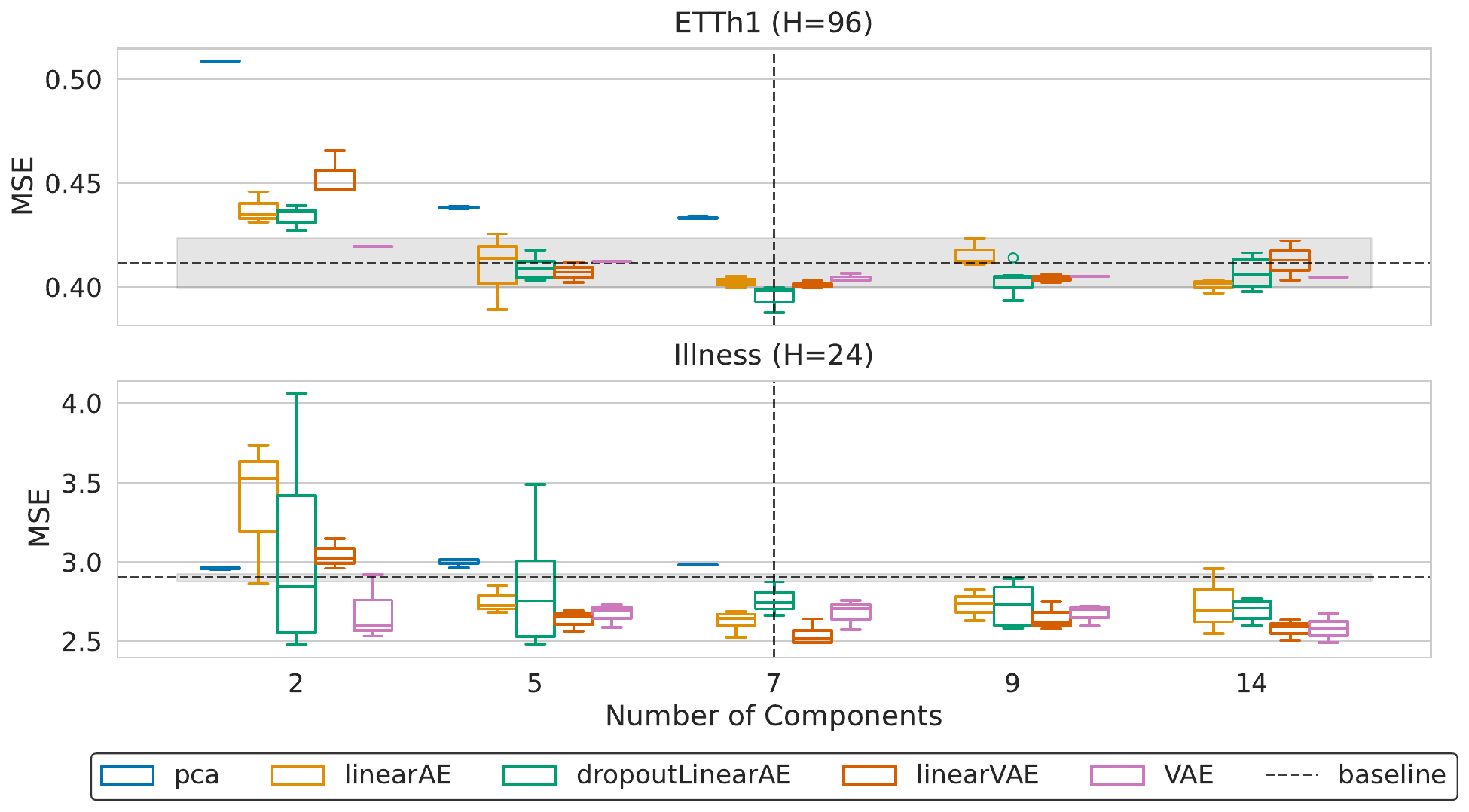}
\caption{Impact of the number of components on model performance. The dashed line indicates $\moment$ performance without adapters, the shaded area its standard deviation, and the vertical line the number of original features.}
\label{fig:dimred}
\end{figure}

\begin{figure*}[t]
\centering
\begin{subfigure}[b]{0.24\textwidth}
         \centering
 \includegraphics[width=\textwidth]{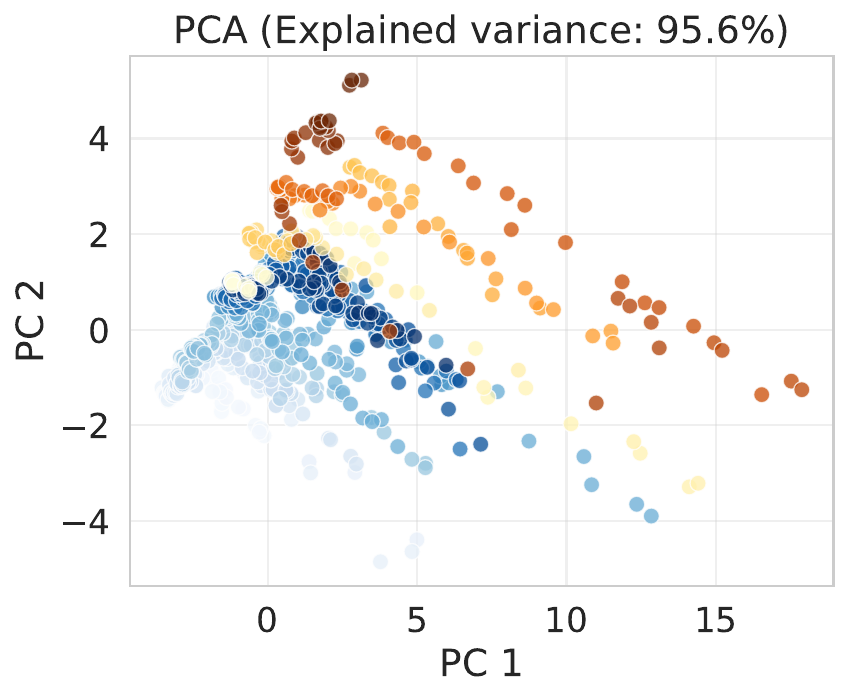}
 \label{fig:subfig_pca}
\end{subfigure}
\hfill
\begin{subfigure}[b]{0.24\textwidth}
 \centering
 \includegraphics[width=\textwidth]{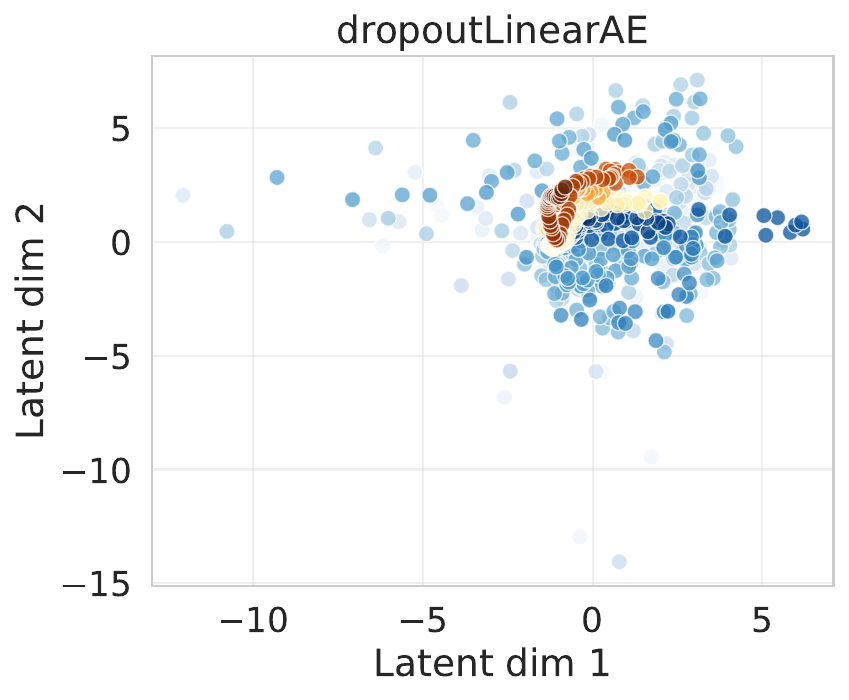}
 \label{fig:subfig_dropout}
\end{subfigure}
\hfill
\begin{subfigure}[b]{0.24\textwidth}
 \centering
 \includegraphics[width=\textwidth]{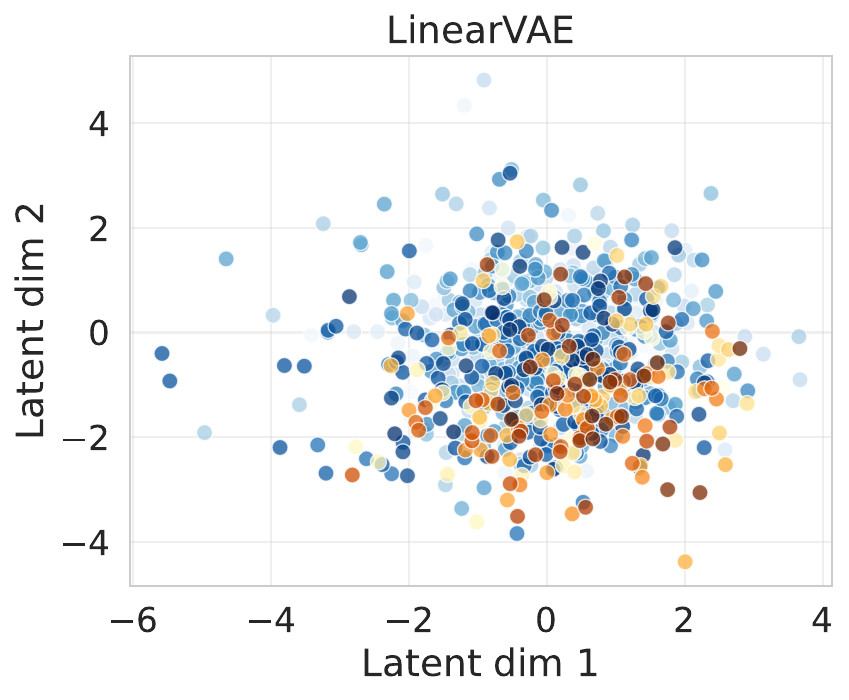}
 \label{fig:subfig_linearvae}
\end{subfigure}
\hfill
\begin{subfigure}[b]{0.24\textwidth}
 \centering
 \includegraphics[width=\textwidth]{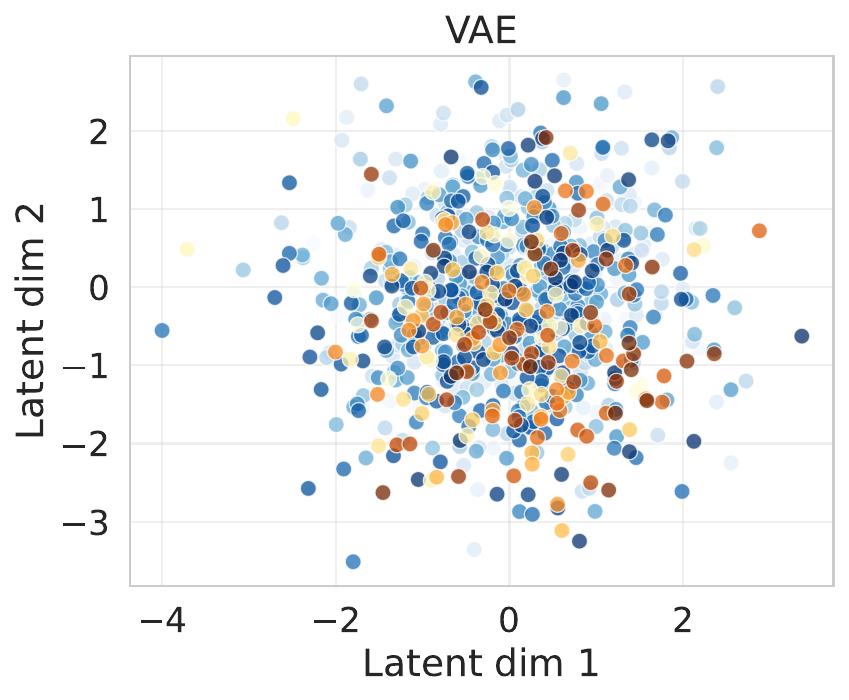}
 \label{fig:subfig_vae}
\end{subfigure}
\vfill
\vskip -0.1in
\begin{subfigure}[b]{0.5\textwidth}
 \centering
 \includegraphics[width=\textwidth]{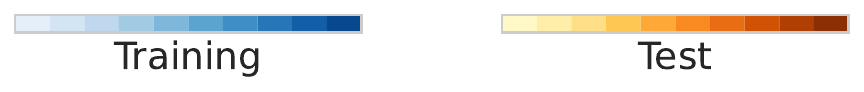}
\end{subfigure}
\caption{Visualization of the latent representation obtained by different adapters \rebuttal{(with number of components equal to 2)} on Illness($H=24$). Shaded colors indicate the time dimension, with lighter colors representing earlier timesteps.}
\label{fig:latent}
\end{figure*}

\noindent\textbf{\adapts improves the performance of \moment.} We present the forecasting error measured by the Mean Squared Error (MSE) in \cref{table:results} and the Mean Absolute Error (MAE) in \cref{appendix:results}. On the ETTh1 dataset with a prediction horizon $H=96$, all adapter-based variants outperform the baseline \moment model, with $\texttt{dropoutLinearAE}$ achieving the best performance, showing an 8\% improvement. Similar results are observed for the Illness dataset, where all adapters improve over the baseline. Notably, the $\texttt{VAE}$ achieves a significant 15\% improvement, reducing the MSE from $2.902$ to $2.461$ at $H=24$. In the Weather dataset, the $\texttt{dropoutLinearAE}$ adapter shows the best improvement across all adapter architectures for $H=96$, while its deterministic counterpart, $\texttt{LinearAE}$, takes the lead at $H=192$. The results on the ExchangeRate dataset are mixed, with some adapters matching the baseline performance ($\texttt{dropoutLinearAE}$ at $H=96$) while others show degraded performance, particularly at a longer prediction horizon ($H=192$), which is also observed for the ETTh1 dataset. Overall, \adapts improves the forecasting accuracy of \moment in 5 out of the 8 considered tasks, matches its performance in 2, and degrades performance in 1 task.

\subsection{Dimensionality Reduction}

\cref{fig:dimred} illustrates the impact of varying latent space dimensions on forecasting performance across different adapters. For the ETTh1 dataset with a $96$-step horizon, all adapter architectures achieve optimal performance at 7 components (matching the original feature count), with MSE values consistently lower than the baseline. Notably, at just 5 components, all adapters (except the \pca baseline) match the baseline score, demonstrating the suitability of our framework for low-resource setups through dimensionality reduction. The Illness dataset ($H=24$) presents more compelling results, as the \vae adapter achieves significantly optimal performance with only 2 components, underscoring the potential of our approach for cost-effective adaptation of time series foundation models. Ultimately, we find that expanding dimensionality beyond the original feature count does not yield further improvements, as no adapter shows notable enhancements past this point.

\subsection{Interpretability of the latent representations}

\cref{fig:latent} compares the representation learning capabilities of different adapters on the Illness($H=24$) dataset, focusing on their ability to distinguish between training and test data. To visualize the raw dataset, we employ PCA for dimensionality reduction, retaining only two principal components, which is justified by the 95.6\% explained variance. When representing the training and test datasets in the space of the first two principal components, we observe a clear distribution shift, potentially complicating the forecasting task for the baseline foundational model. In contrast, using \adapts results in well-overlapping Gaussian distributions for the training and test data in the latent space. This demonstrates our framework's ability to enforce a structured, isotropic representation that mitigates distribution shift. This effect is particularly pronounced with the \vae adapter and, to a lesser extent, with $\texttt{LinearVAE}$ and $\texttt{dropoutLinearAE}$.

The findings emphasize the advantages of \vae in managing distribution shift, a critical challenge in time series representation learning. By modeling uncertainty and enforcing a continuous latent space, \vae enhance generalization, making them especially valuable for real-world applications where test distributions differ from training data. This aligns with the objective of utilizing adapters in foundational models to optimize zero-shot performance, ensuring robustness across various tasks without extensive fine-tuning.

\subsection{On the calibration of the probabilistic adapters}

\begin{wrapfigure}{r}{0.2\textwidth}
  \begin{center}
  \vskip -0.1in
    \includegraphics[width=0.2\textwidth]{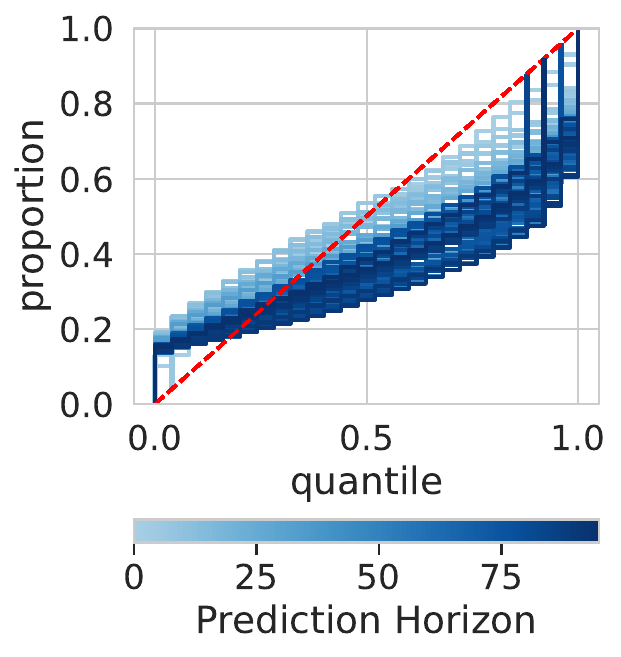}
  \end{center}
  \caption{Reliability diagram for the first feature of the ETTh1 ($H=96$) dataset using $\texttt{LinearVAE}$.}
  \label{fig:calibration}
\end{wrapfigure}

To evaluate the calibration of our adapter-based probabilistic forecasters, we use quantile calibration as depicted in the reliability diagram in \cref{fig:calibration}. In an ideal scenario, a well-calibrated probabilistic forecast should align with the red dashed diagonal, indicating that the empirical proportion of observations falls within the predicted quantiles at the expected rate. The overall conclusion is that we observe a gradual deviation from ideal calibration as the prediction horizon increases (darker shades). While early prediction horizons display reasonably well-calibrated predictions, longer-horizon forecasts systematically underestimate uncertainty, as shown by the curve falling below the diagonal. This indicates that observed values exceed predicted quantiles more frequently than expected, suggesting that the predictive distribution becomes too narrow, resulting in overconfident forecasts.

\subsection{Ablation studies}

\noindent\textbf{Influence of $\sigma$ and $\beta$ in the \vae Adapter.} \cref{fig:beta_sigma} illustrates an ablation study examining the $\beta$ parameter in $\beta$-\vae and the noise scale $\sigma$ of the likelihood model applied to the prediction $\hat{\bY}$, assessing their effects on MSE and Expected Calibration Error (ECE). The MSE heatmap (left) demonstrates that increasing $\beta$ generally diminishes MSE, with the lowest values observed at $\beta = 2.0$ and $\beta = 4.0$, particularly for higher $\log \sigma^2$. This indicates that stronger regularization through $\beta$ can enhance forecasting accuracy, possibly due to the disentangling effect of regularization towards a prior distribution with statistically independent components. Conversely, the ECE heatmap (right) shows that higher $\beta$ and $\log \sigma^2$ values result in lower calibration error, with optimal results at $\beta = 4.0$ and $\log \sigma^2 = 3.0$. This outcome is anticipated, as larger values of $\beta$ and $\sigma$ mitigate overfitting, where the model tends to exhibit overconfidence in its predictions. Additionally, it is observed that maintaining a fixed $\sigma$ during training generally outperforms including it in the optimization loop, a configuration denoted as \emph{auto} in \cref{fig:beta_sigma}.

\begin{figure}[ht]
\centering
\includegraphics[width=\columnwidth]{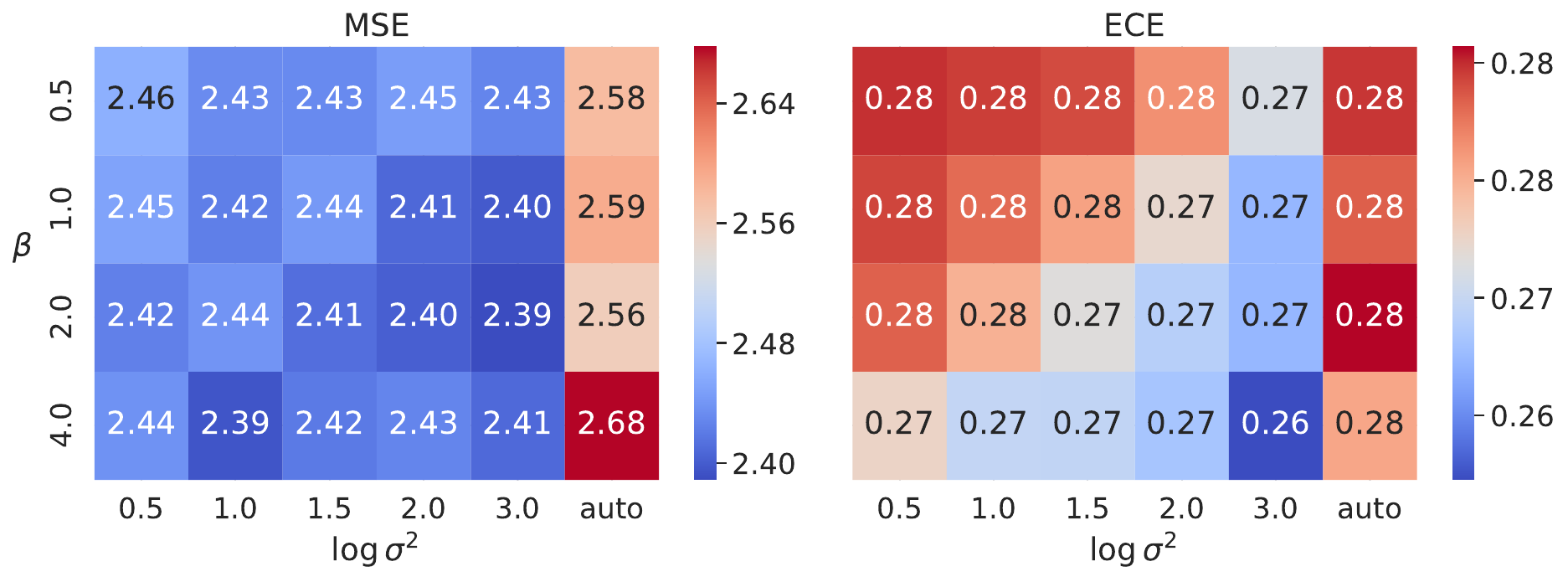}
\caption{$\beta$ and $\log \sigma^2$ \vae hyperparameters ablation on the Illness($H=24$) dataset. For reference, the \moment baseline score on this task is $2.902_{\pm 0.023}$.}
\label{fig:beta_sigma}
\end{figure}

\noindent\textbf{$\texttt{LinearAE}$ components.} The ablation study presented in \cref{fig:enc_dec} examines the performance of different components of the linear autoencoder adapter ($\texttt{LinearAE}$) across three datasets: ETTh1, Weather, and ExchangeRate. The figure compares the full linear autoencoder with its encoder-only ($\texttt{LinearEncoder}$) and decoder-only ($\texttt{LinearDecoder}$) variants. Overall, the results reveal that the decoder component of the linear autoencoder plays the most important role in minimizing the forecasting error across all datasets. The encoder-only variant's contribution varies, being more impactful in the Weather dataset compared to ETTh1 and ExchangeRate. These findings highlight the significance of the decoder in the $\texttt{LinearAE}$ adapter and suggest that, in the deterministic case, a decoder might be sufficient to capture feature dependencies. 

\begin{figure}[ht]
\centering
\includegraphics[width=\columnwidth]{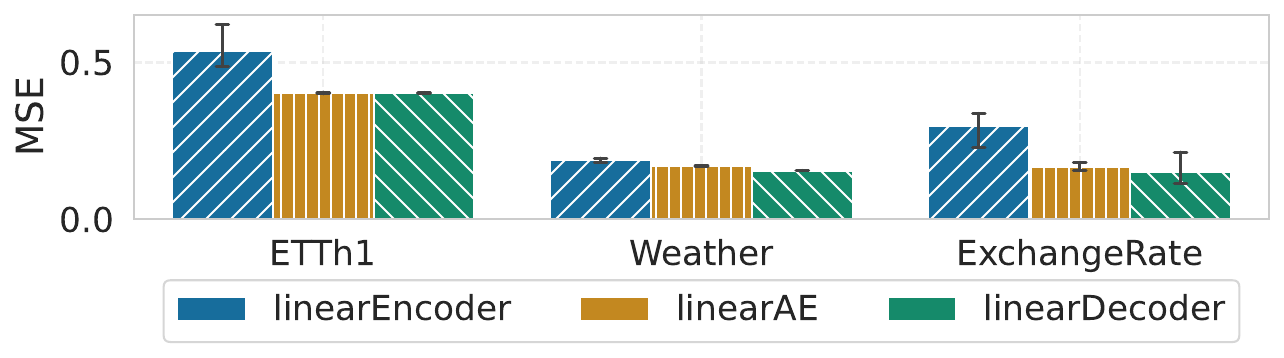}
\caption{LinearAE components ablation.}
\label{fig:enc_dec}
\end{figure}

Nevertheless, as shown in our previous experiments, particularly \cref{table:results}, probabilistic adapters generally outperformed the deterministic ones. This underscores the importance of the encoder as well, which is responsible for approximating the posterior distribution in the latent space—a mechanism inherent to our probabilistic framework.

\section{Conclusion}
\label{sec:discussion}

In this paper, we investigate how pre-trained univariate time series foundation models can be adapted for probabilistic multivariate forecasting. To address this challenge, we introduce the \adapts framework. Our method offers a novel approach to training feature space transformations that facilitate uncertainty quantification and enhance the performance of baseline foundation models. Through a series of experiments, we demonstrate that our framework improves forecasting accuracy, provides reasonably well-calibrated uncertainty estimates, reduces inference cost through dimensionality reduction, and offers interpretable feature space latent representations.

\noindent\textbf{Limitations \& Future directions.} Our work establishes a principled framework for adapting pre-trained univariate foundation models to multivariate probabilistic time series forecasting. 
While we focus on \moment, our approach can be applied on other univariate deterministic FMs; we leave this direction to future work.
Additionally, while variational inference is chosen for its efficiency, exploring alternative methods like Markov Chain Monte Carlo could improve uncertainty estimation, albeit at a higher computational cost.

Another promising direction is refining the calibration of our uncertainty estimates. While our framework flexibly extracts predictive uncertainty, investigating theoretical guarantees and recalibration techniques could further enhance reliability.

Overall, our work lays a strong foundation for efficiently adapting FMs, extending their applicability to broader forecasting challenges while preserving their expressive power.


\section*{Reproducibility Statement}
In order to ensure reproducibility we will release the code at \href{https://github.com/abenechehab/AdaPTS}{https://github.com/abenechehab/AdaPTS}, once the paper has been accepted. The implementation details and hyperparameters are listed in \cref{appendix:implem}.

\bibliography{main}
\bibliographystyle{main}

\newpage
\appendix
\onecolumn

\textbf{\LARGE Appendix}

\paragraph{Outline.} In~\cref{appendix:theory}, we provide the proofs and a discussion on \cref{prop:solution} and \cref{prop:vae}. We then provide a perspective on Normalizing Flows as adapters in \cref{appendix:flows}. The experimental setup is presented in \cref{appendix:exp_setup}, including all the implementation details in \cref{appendix:implem}. Finally we showcase some additional results in \cref{appendix:results}.  

\addtocontents{toc}{\protect\setcounter{tocdepth}{2}}

\renewcommand*\contentsname{\Large Table of Contents}

\tableofcontents
\clearpage

\section{Theoretical analysis}
\label{appendix:theory}

\subsection{Proof of \cref{prop:solution}}
\label{appendix:solution_proof}
We begin by restating the proposition, and its underlying assumptions:

\begin{assumption}
$\bW_{\varphi}$ has full rank: $\text{rank}(\bW_{\varphi}) = D$, insuring its invertibility.
\end{assumption}

\begin{assumption}
For ease of derivation, we consider a similar linear parametrization for the foundation model: $\fm(\bX) =  \bW_{\text{FM}}^\top \bX + \vb_{\text{FM}} \mbf{1}^\top$ where $\bW_{\text{FM}} \in \RR^{L \times H}$, $\vb_{\text{FM}} \in \RR^H$, and $\mbf{1}$ a vector of ones of dimension $D$.
\end{assumption}

\begin{proposition}[Optimal linear adapter]
Under \cref{ass:invertible} and \cref{ass:fm}, the \emph{closed-form} solution of the problem: 

\begin{equation}
    \Loss (\bW_{\varphi}) = \|\bY - \big( \bW_{\text{FM}}^\top \bX \bW_{\varphi} + \vb_{\text{FM}} \mbf{1}^\top \big) \bW_{\varphi}^{-1}\|_F^2
\end{equation}

writes as:

\begin{equation}
    \bW_{\varphi}^* = (\bB^\top \bA)^{+} \bB^\top \bB
\end{equation}

where $\bW_{\varphi}^* = \argmin_{\bW_{\varphi} \in \mathcal{GL}_D(\RR)} \Loss (\bW_{\varphi})$, $\bA = \bY -  \bW_{FM}^\top \bX$, $\bB = \vb_{FM} \mbf{1}^\top$, and $(\bB^{\top} \bA)^{+}$ denoting the \emph{pseudo-inverse} operator.

\end{proposition}

\begin{proof}

We begin by expanding the loss function:
\begin{align*}
    \Loss (\bW_{\varphi}) &= \|\bY - (\bW_{\text{FM}}^\top \bX \bW_{\varphi} + \vb_{\text{FM}} \mbf{1}^\top) \bW_{\varphi}^{-1} \|_F^2 \\
    &= \| \bA - \bB \bW_{\varphi}^{-1} \|_F^2
\end{align*}

where $\bA = \bY -  \bW_{FM}^\top \bX$ and $\bB = \vb_{FM} \mbf{1}^\top$. Expanding the Frobenius norm:
\begin{equation*}
    \Loss (\bW_{\varphi}) = \text{Tr} \left( (\bA - \bB \bW_{\varphi}^{-1})^\top (\bA - \bB \bW_{\varphi}^{-1}) \right)
\end{equation*}

Taking the gradient with respect to $\bW_{\varphi}^{-1}$ yields:

\begin{equation*}
    \frac{\partial \Loss}{\partial \bW_{\varphi}^{-1}} = - 2 \bB^\top \bA + 2 \bB^\top \bB \bW_{\varphi}^{-1}
\end{equation*}

Knowing that $\bW_{\varphi}$ is invertible, We have that: $ \frac{\partial \Loss}{\partial \bW_{\varphi}} =  - \bW_{\varphi}^{-\top} \frac{\partial \Loss}{\partial \bW_{\varphi}^{-1}} \bW_{\varphi}^{-\top}$

hence

\begin{equation*}
    \frac{\partial \Loss}{\partial \bW_{\varphi}} = -2 \bW_{\varphi}^{-T} \left( \bB^\top \bA - \bB^\top \bB \bW_{\varphi}^{-1} \right) \bW_{\varphi}^{-T}.
\end{equation*}

Setting $\frac{\partial \Loss}{\partial \bW_{\varphi}} = 0$ and multiplying both sides by $\bW_{\varphi}^\top$, we obtain:
\begin{equation*}
    \bB^\top \bA = \bB^\top \bB \bW_{\varphi}^{-1}.
\end{equation*}
Multiplying both sides by $\bW_{\varphi}$:
\begin{equation*}
     \bB^\top \bA \bW_{\varphi} = \bB^\top \bB.
\end{equation*}
Finally applying the pseudo-inverse to solve for $\bW_{\varphi}$ gives our final result:
\begin{equation*}
    \bW_{\varphi}^* = (\bB^\top \bA)^{+} \bB^\top \bB.
\end{equation*}

Given the convexity of $\Loss (\bW_{\varphi})$ (which follows from the convexity of the Frobenius norm $\norm{\cdot}_{\mathrm{F}}^2$, the inverse operation, and an affine transformation), we conclude that $\bW_{\varphi}^*$ is a global solution for \cref{eq:prob2}.

\begin{remark}
    We make use of the pseudo-inverse due to the current construction of the matrix $\bB$ (with identical rows) which implies that the product $\bB^\top \bA$ is degenerate. To bypass this limitation and further ensure the invertibility of $\bW_{\varphi}^*$, we can revisit the definition of the foundation model in \cref{ass:fm} to include channel dependent biases and ensure a full rank matrix $\bB$.
\end{remark}

\end{proof}
    
\subsection{Proof of \cref{prop:vae}}
\label{appendix:vae_proof}

To derive the evidence lower bound (ELBO) used in the training objective of the VAE adapter, we start from the marginal likelihood of the observed data $\bY$ given the inputs $\bX$ and foundation model $\fm$. The marginal likelihood is expressed as:
\begin{equation}
\log p_\theta(\bY|\bX, \fm) = \log \int p_\theta(\bY|\bX, \fm(\bZ)) p(\bZ) \, d\bZ,
\end{equation}
where $\bZ$ is the latent variable, $p_\theta(\bY|\bX, \fm(\bZ))$ is the likelihood model parameterized by $\theta$, and $p(\bZ)$ is the prior distribution over the latent variable $\bZ$.

Direct optimization of this marginal likelihood is generally intractable due to the integration over $\bZ$. To make this optimization feasible, we introduce a variational distribution $q_\phi(\bZ|\bX)$, parameterized by $\phi$, as an approximation to the true posterior $p_\theta(\bZ|\bX, \bY, \fm)$. Using $q_\phi(\bZ|\bX)$, we can reformulate the log-marginal likelihood as follows:
\begin{align}
\log p_\theta(\bY|\bX, \fm) 
&= \log \int q_\phi(\bZ|\bX) \frac{p_\theta(\bY|\bX, \fm(\bZ)) p(\bZ)}{q_\phi(\bZ|\bX)} \, d\bZ \\
&= \log \mathbb{E}_{q_\phi(\bZ|\bX)} \left[ \frac{p_\theta(\bY|\bX, \fm(\bZ)) p(\bZ)}{q_\phi(\bZ|\bX)} \right].
\end{align}

Using Jensen's inequality, we can derive a lower bound on this log-marginal likelihood:
\begin{align}
\log p_\theta(\bY|\bX, \fm) 
&\geq \mathbb{E}_{q_\phi(\bZ|\bX)} \left[ \log \frac{p_\theta(\bY|\bX, \fm(\bZ)) p(\bZ)}{q_\phi(\bZ|\bX)} \right] \\
&= \mathbb{E}_{q_\phi(\bZ|\bX)} \left[ \log p_\theta(\bY|\bX, \fm(\bZ)) \right] 
    - \mathbb{E}_{q_\phi(\bZ|\bX)} \left[ \log \frac{q_\phi(\bZ|\bX)}{p(\bZ)} \right].
\end{align}

The second term can be rewritten as the Kullback-Leibler (KL) divergence between the variational posterior $q_\phi(\bZ|\bX)$ and the prior $p(\bZ)$:
\begin{equation}
\mathrm{KL}\left(q_\phi(\bZ|\bX) \,\|\, p(\bZ)\right) = \mathbb{E}_{q_\phi(\bZ|\bX)} \left[ \log \frac{q_\phi(\bZ|\bX)}{p(\bZ)} \right].
\end{equation}

Substituting this into the inequality, we obtain the evidence lower bound (ELBO):
\begin{equation}
\log p_\theta(\bY|\bX, \fm) \geq \mathbb{E}_{q_\phi(\bZ|\bX)} \left[ \log p_\theta(\bY|\bX, \fm(\bZ)) \right] 
    - \mathrm{KL}\left(q_\phi(\bZ|\bX) \,\|\, p(\bZ)\right).
\end{equation}

The ELBO consists of two terms:
\begin{itemize}
    \item The \emph{forecasting} term, $\mathbb{E}_{q_\phi(\bZ|\bX)} \left[ \log p_\theta(\bY|\bX, \fm(\bZ)) \right]$, which measures how well the model can reconstruct $\bY$ given the latent variable $\bZ$.
    \item The \emph{regularization} term, $\mathrm{KL}\left(q_\phi(\bZ|\bX) \,\|\, p(\bZ)\right)$, which encourages the variational posterior to stay close to the prior distribution $p(\bZ)$.
\end{itemize}

Thus, maximizing the ELBO provides a tractable way to train the parameters $\theta$ and $\phi$ by optimizing the balance between forecasting accuracy and latent space regularization.
\qed

\section{Normalizing Flows}
\label{appendix:flows}

Normalizing Flows make use of invertible transformations to map a simple base distribution (e.g. Gaussian) to a complex data distribution. Each transformation $ T $ is designed to maintain invertibility and efficient Jacobian computation. The transformation is applied iteratively: $ \bZ = T_k \circ T_{k-1} \circ \dots \circ T_1(\bX) $. Current Normalizing Flow instantiations (e.g. \texttt{RealNVP}) make use of generic invertible transformations such as \emph{coupling flows}; the latters can be parametrized using a neural network leading to powerful non-linear generative models that are trained to maximize the data log-likelihood:
\[
\log p(\bX) = \log p(\bZ) + \sum_{i=1}^k \log \left| \det \frac{\partial T_i(\cdot;\theta)}{\partial \bZ_{i-1}} \right|
\]
where $\theta$ denote the parameters of the non-linear parametrization of the invertible transformations $T_i$, and $\bZ_{i-1}$ is the output of the transformation $T_{i-1}$.
In the context of time series adapters, we directly optimize the parameters of the transformations based on their direct and inverse application on the time series forecasting problem:

\begin{equation*}
\begin{split}
\Loss_{\text{flow}} = \| \bY - &T^{-1}_{1} \circ T^{-1}_{2} \circ \dots \circ T^{-1}_k( \\ &\fm \big( T_k \circ T_{k-1} \circ \dots \circ T_1(\bX; \theta) \big); \theta) \|_F^2
\end{split}
\end{equation*}

where the encoder is represented by the series of direct transformations: $\enc(\cdot) = T_k \circ T_{k-1} \circ \dots \circ T_1(\cdot; \theta)$, and respectively the decoder by the series of inverse transformations $\dec(\cdot) = T^{-1}_{1} \circ T^{-1}_{2} \circ \dots \circ T^{-1}_k(\cdot; \theta)$.

As defined here, Normalizing Flows suffer from the constraint of keeping the same dimension in both original and learned representation space. For this purpose, we investigate coupling a normalizing flow with a linear encoder-decoder type of architecture to enable dimensionality reduction prior to applying the transformation $T_i$. The parameters of the additional encoder and decoder are then jointly trained to optimize the learning objective $\Loss_{\text{flow}}$.

Given that the parameters of the encoder and the decoder are shared in Normalizing Flows, the gradient-based optimization within our framework receives conflicting directions due to gradient flow from both the direct and inverse transformations simultaneously. We discovered that this adapter construction was challenging to optimize in practice, and we defer the exploration of this direction to future research endeavors.

\section{Experimental setup}
\label{appendix:exp_setup}

\subsection{Datasets}
\label{appendix:datasets}

\begin{table}[htbp]
    \centering
    \caption{Characteristics of the multivariate time series datasets used in our experiments with various sizes and dimensions.}
    \label{tab:dataset_description}
    \scalebox{1}{
    \begin{tabular}{lcccc}
        \toprule
         Dataset & ETTh1 & Illness & ExchangeRate & Weather\\
         \midrule
         \# features & $7$ & $7$ & $8$ & $21$\\
         \# time steps & $13603$ & $169$ & $6791$ & $51899$ \\
         Granularity & 1 hour & 1 week & 1 day & 10 minutes \\
         (Train, Val, Test) & (8033, 2785, 2785) & (69, 2, 98)  & (4704, 665, 1422) & (36280, 5175, 10444) \\
         \bottomrule
    \end{tabular}
    }
\end{table}

\subsection{Implementation details}
\label{appendix:implem}

In this section, we describe the full \adapts framework, starting from the data preprocessing, the training algorithm, and the hyperparameters optimization.

\paragraph{Preprocessing.} Given that the adapter as defined in \cref{def:adapter} is a feature space transformation, we start by rescaling (\emph{StandardScaler} and \emph{MinMaxScaler}) the data where all the timesteps are regarded as data points. To account for the temporal specificities in each batch, we use Reversible Instance Normalization (RevIn) \citep{kim2022reversible} that has been proven to mitigate time-related distribution shifts in time series problems. Finally, and following the observation that $\pca$ when composed with a linear adapter showed the best result in the case of correlated data (\cref{fig:linear_synthetic}), we include the possibility of applying \emph{full-component} $\pca$ as part of our data pre-processing pipeline.

\paragraph{Training parameters.} After the pre-processing phase, we proceed to split the data into a \emph{train-validation-test} sets, where the validation set serves as a tool to select the best hyperparameters for the adapter. The resulting adapter that is instantiated with the optimal hyperparameters is then tested against the unseen test dataset. For all of our experiments, we first train the linear forecasting head of \moment (referred to as \emph{Linear Probing} in \cite{goswami2024moment}) with the Adam optimizer \citep{kingma2017adammethodstochasticoptimization}, a batch size of $32$, a one cycle scheduler starting with $0.001$ as learning rate. Once the forecasting linear head is trained, we freeze its parameters and proceed to training the adapter. This is done using the Adam optimizer, a batch size of $32$, a reduce on plateau scheduler starting with $0.001$ as learning rate.


\paragraph{Hyperparameter optimization.} In order to select the best hyperparameters for the adapter architecture we use \emph{Ray tune}~\citep{liaw2018tuneresearchplatformdistributed} with the Heteroscedastic and Evolutionary Bayesian Optimisation solver (HEBO) \citep{Cowen-Rivers2022-HEBO} engine, reporting the average mean squred error (MSE) from \emph{k-fold} cross validation. \cref{table:hypers} shows the default hyperparameters for each considered adapter.

\begin{table}[!ht]
  \scriptsize
  \caption{Adapters hyperparameters.}
  \label{table:hypers}
  \centering
  \begin{tabular}{lllll}
    \toprule
    adapter
    & \texttt{LinearAE}
    & \texttt{DropoutLinearAE}
    & \texttt{LinearVAE}
    & \texttt{VAE}
    \\
    \cmidrule(r){1-5}
    p dropout
    & $-$   
    & $0.1$ 
    & $-$ 
    & $-$ 
    \\
    Number of layers
    & $-$   
    & $-$ 
    & $-$ 
    & $2$ 
    \\
    Hidden dimension
    & $-$   
    & $-$ 
    & $-$ 
    & $128$
    \\
    $\beta$
    & $-$   
    & $-$ 
    & $0.5$
    & $0.5$
    \\
    $\sigma$
    & $-$   
    & $-$ 
    & $1.0$
    & $1.0$
    \\
    \bottomrule
  \end{tabular}
\end{table}

\section{Additional results}
\label{appendix:results}

\subsection{$\moment$ applied to synthetic data.}
\label{appendix:moment_synth}

To validate the adapter \emph{optimality} condition with large non-linear foundation models, we use $\moment$ \citep{goswami2024moment}. The optimal linear adapter in this case minimizes the following intractable objective:

\begin{equation}
\label{eq:moment_linear}
    \Loss (\bW_{\varphi}) = \|\bY - f_{\text{$\moment$}}\big(\bX \bW_{\varphi} \big) \bW_{\varphi}^{-1}\|_{\mathrm{F}}^2
\end{equation}

To approximately solve this optimization problem, we instantiate $\bW_{\varphi}$ as a single-linear-layer encoder denoted $\enc_\theta$, and respectively the inverse transformation $\bW_{\varphi}^{-1}$ as a single-linear-layer decoder denoted $\dec_\theta$. We then use gradient-based optimization of the parameters $\theta$ using the Adam optimizer, aiming at solving the following optimization problem:

\begin{equation}
\label{eq:theta}
    \theta^* = \argmin_\theta \|\bY - \dec_\theta \big( f_{\text{$\moment$}}(\enc_\theta (\bX)) \big)\|_{\mathrm{F}}^2
\end{equation}

\begin{figure}
     \centering
     \begin{subfigure}[b]{0.49\textwidth}
         \centering
        \includegraphics[width=\textwidth]{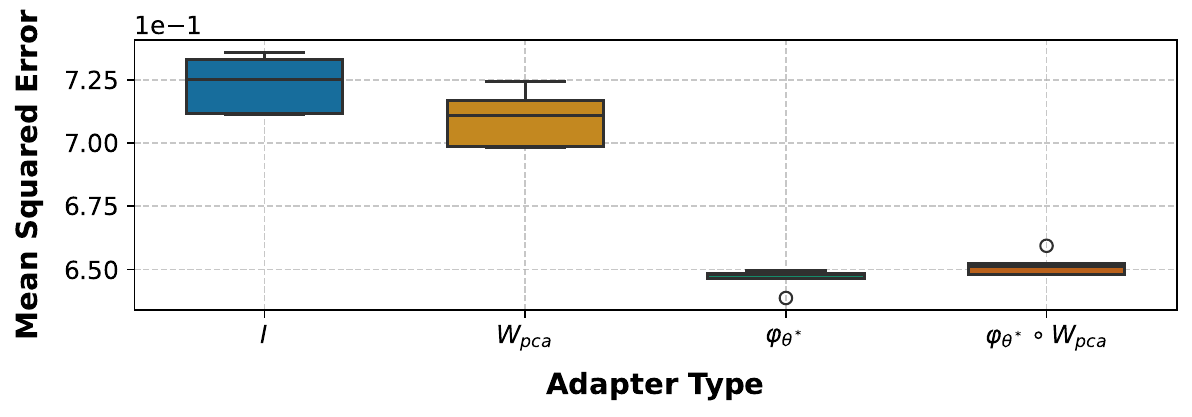}
         \caption{Independent}
         \label{fig:indp}
     \end{subfigure}
     \hfill
     \begin{subfigure}[b]{0.49\textwidth}
         \centering
        \includegraphics[width=\textwidth]{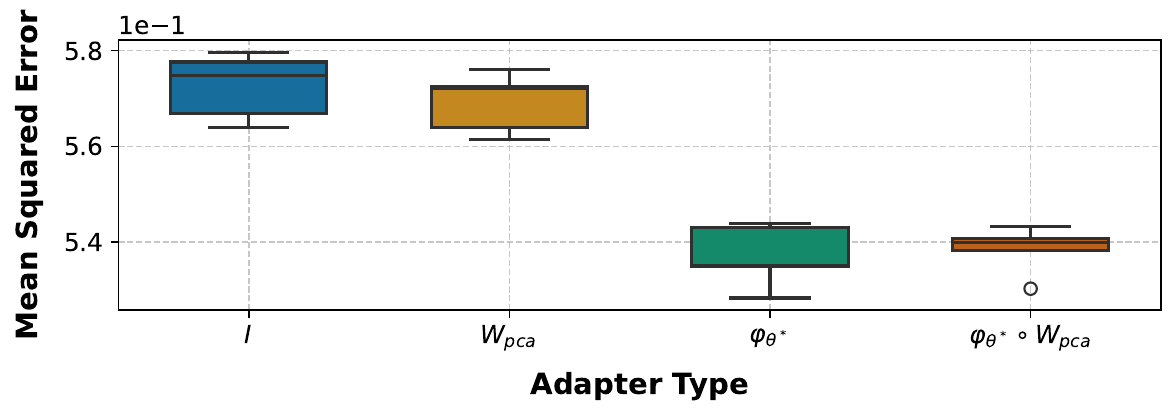}
         \caption{Correlated}
         \label{fig:corr}
     \end{subfigure}
        \caption{$\moment$ on simulated independent data.}
        \label{fig:moment_synthetic}
\end{figure}

\cref{fig:moment_synthetic} shows the performance gain obtained by optimizing a linear adapter on \emph{$\moment$-small} foundation model. Unlike the tractable case, we observe that in both data modalities (independent and correlated data), $\pca$ has little to no improvement over the identity baseline, while $\varphi_{\theta^*}$ reaches an order of magnitude better solution. This confirms our intuition about the existence of a better solution than the identity matrix, even in the case of real-world complex foundation models.



\subsection{Mean Absolute Error}

\begin{table*}[ht]
    \centering
    \begin{tabular}[t]{cccccccc}
        \toprule[\thick pt]
        \multirow{2}{*}{Dataset} & \multirow{2}{*}{H} & \multicolumn{1}{c}{No adapter} & \multicolumn{5}{c}{with adapter} \\
        \cmidrule(lr){3-3} 
        \cmidrule(lr){4-8}
         & & $\text{Moment}_{\text{small}}$ & pca & linear & dropout & linear VAE & VAE \\
        \midrule[\thick pt]
        \multirow{2}{*}{ETTh1} & 96 & $0.422_{\pm 0.006}$ & $0.440_{\pm 0.000}$ & $0.423_{\pm 0.003}$ & $\mathbf{0.415_{\pm 0.002}}$ & $0.420_{\pm 0.001}$ & $0.426_{\pm 0.001}$ \\
         & 192  &  $\mathbf{0.436_{\pm 0.000}}$ & $0.445_{\pm 0.000}$ & $0.449_{\pm 0.003}$ & $0.450_{\pm 0.001}$ & $0.451_{\pm 0.001}$ & $0.444_{\pm 0.001}$  \\
        \cmidrule(lr){1-8}
        \multirow{2}{*}{Illness} & 24 & $1.143_{\pm 0.007}$ & $1.163_{\pm 0.001}$ & $2.624_{\pm 0.035}$ & $1.156_{\pm 0.016}$ & $1.074_{\pm 0.011}$ & $\mathbf{1.057_{\pm 0.012}}$ \\
         & 60  & $1.149_{\pm 0.001}$ & $1.161_{\pm 0.001}$ & $1.227_{\pm 0.030}$ & $1.173_{\pm 0.015}$ & $1.112_{\pm 0.021}$ & $\mathbf{1.105_{\pm 0.021}}$ \\
         \cmidrule(lr){1-8}
        \multirow{2}{*}{Weather} & 96 & $0.232_{\pm 0.010}$ & $0.235_{\pm 0.000}$ & $0.226_{\pm 0.000}$ & $0.212_{\pm 0.001}$ & $\mathbf{0.218_{\pm 0.001}}$ & $0.243_{\pm 0.001}$ \\
         & 192 & $\mathbf{0.251_{\pm 0.001}}$ & $0.260_{\pm 0.001}$ & $\mathbf{0.251_{\pm 0.001}}$ & $\mathbf{0.251_{\pm 0.000}}$ & $0.255_{\pm 0.000}$ & $0.274_{\pm 0.000}$ \\
        \cmidrule(lr){1-8}
        \multirow{2}{*}{ExchangeRate} & 96 & $\mathbf{0.252_{\pm 0.010}}$ & $0.264_{\pm 0.000}$ & $0.308_{\pm 0.010}$ & $0.269_{\pm 0.012}$ & $0.376_{\pm 0.031}$ & $0.488_{\pm 0.003}$ \\
         & 192 & $\mathbf{0.329_{\pm 0.001}}$ & $0.335_{\pm 0.000}$ & $0.415_{\pm 0.002}$ & $0.419_{\pm 0.010}$ & $0.513_{\pm 0.010}$ & $0.585_{\pm 0.008}$ \\
        \bottomrule[\thick pt]
    \end{tabular}
    \caption{Performance comparison between the baseline $\moment$ model without adapters against different adapter architectures (\pca, $\texttt{LinearAE}$, $\texttt{dropoutLAE}$, $\texttt{LinearVAE}$, and \vae), for multivariate long-term forecasting with different horizons $H$. We display the average test MAE $\pm$ standard error obtained on $3$ runs with different seeds. \textbf{Best} results are in bold, with lower values indicating better performance.}
    \label{table:results_mae}
\end{table*}







\end{document}